\def\year{2022}\relax
\pdfoutput=1
\documentclass[letterpaper]{article} %
\usepackage{aaai22}  %
\usepackage[T1]{fontenc} \usepackage{ae} \usepackage{aecompl}
\usepackage{times}  %
\usepackage{helvet}  %
\usepackage{courier}  %
\usepackage{xurl}
\usepackage[hidelinks]{hyperref}  %
\hypersetup{
  breaklinks=true,   %
  pdfusetitle=true,  %
}
\usepackage{graphicx} %
\usepackage{natbib}  %
\usepackage{caption} %
\DeclareCaptionStyle{ruled}{labelfont=normalfont,labelsep=colon,strut=off} %
\usepackage{algorithm}
\usepackage{algorithmic}

\usepackage{amsmath}    %
\usepackage{amssymb}    %
\usepackage{amsthm}     %
\usepackage{bm}         %
\usepackage{thmtools}
\usepackage{symbols}
\usepackage{symbols_extra}
\usepackage{influence-diagrams}
\usepackage{cancel}
\usepackage[capitalize,noabbrev]{cleveref}
\usepackage{multicol}
\usepackage{multirow}
\usepackage{hhline}

\usepackage{newfloat}
\usepackage{listings}
\floatstyle{ruled}
\newfloat{listing}{tb}{lst}{}
\floatname{listing}{Listing}
\title{
Why Fair Labels Can Yield Unfair Predictions:\\ Graphical Conditions for Introduced Unfairness
}
\author{
    Carolyn Ashurst,\textsuperscript{\rm 1}
    Ryan Carey,\textsuperscript{\rm 2}
    Silvia Chiappa,\textsuperscript{3}
    Tom Everitt\textsuperscript{\rm 3}
}
\affiliations{
    \textsuperscript{\rm 1}Alan Turing Institute,
    \textsuperscript{\rm 2}University of Oxford,
    \textsuperscript{\rm 3}DeepMind \\
    cashurst@turing.ac.uk, ry.duff@gmail.com, csilvia@deepmind.com, tomeveritt@deepmind.com
}

\begin{document}

\maketitle

\begin{abstract}
In addition to reproducing  discriminatory relationships in the training data,
machine learning systems can also introduce or amplify discriminatory effects.  We  refer  to  this as \emph{introduced  unfairness},  and  investigate the conditions under which it may arise. To this end, we propose \emph{introduced total variation} as a measure of introduced unfairness, and establish graphical conditions under which it may be incentivised to occur. 
These criteria imply that adding the sensitive attribute as a feature removes the incentive for introduced variation under well-behaved loss functions. Additionally, taking a causal perspective, \emph{introduced  path-specific effects} shed light on the issue of when specific paths should be considered fair.
\end{abstract}

\section{Introduction}

It is often said that ``unfair data leads to unfair models'', %
because machine learning systems tend to learn biases present in the training data.
However, sometimes a model can produce unfair predictions even when the training labels are fair.
More generally, a model can amplify the unfairness present in training labels.

To quantify this effect, which we refer to as \emph{introduced unfairness}, we propose computing a suitable measure of disparity 
for the training labels and the model predictions, and then comparing the two values.
One such measure is the \emph{total variation} \citep{zhang2018fairness}, 
a generalisation of demographic disparity that describes the strength of the statistical relationship between 
a sensitive variable (such as gender) and an outcome (such as the score given to an applicant's resume).
If the total variation of the predictions is greater than that of the training labels
we say that there is \emph{introduced total variation} (\S \ref{sec:defining-introduced}).

Introduced total variation is distinct from existing measures of unfairness
like \emph{separation} and \emph{sufficiency}, which generalise \emph{equalised odds} and \emph{predictive parity} respectively (\S \ref{sec:sep_suf}).
For binary classifiers, separation prevents introduced total variation, while sufficiency prevents reduced total variation.
In contrast, absence of introduced total variation guarantees neither separation or sufficiency.

Introduced unfairness would seem to be avoidable, since it is never present for a perfectly accurate predictor.
This raises the questions: why and when does introduced unfairness occur, and how can it be removed?
To answer these questions, we use structural causal models and their associated graphs
to represent the relationships between the variables underlying the training data (\S \ref{sec:setup}).
We also build on influence diagrams, by including the predictor and the loss function of the machine learning system in the same graph. This
allows us to reuse results for predictors that are optimal given the available features.

Our key contributions are establishing conditions for introduced total variation, and insight into why it can occur (\S \ref{sec:incentives}).
We find that predictors can be unfair in spite of fair labels because a feature they depend on is statistically dependent with the sensitive attribute.
We also show that the class of loss function influences the conditions under which introduced total variation is incentivised (\S \ref{sec:padm}).
In particular, we consider \emph{P-admissible} loss functions \citep{miller1993loss} --- those for which it is optimal to output the expected label given the input features --- such as mean squared error and cross-entropy.
Predictors that are optimal with respect to a P-admissible loss function can introduce unfairness because they are unable to disentangle information that they have about the sensitive attribute from the information they have about the target label.
Indeed, making the sensitive attribute available as a feature is always enough to prevent introduced total variation being incentivised under P-admissible loss.
We discuss benefits and limitations of this approach to preventing introduced unfairness.

The notion of introduced unfairness can be applied to causal definitions of fairness as readily as statistical ones (\S \ref{sec:psie}).
In particular, \emph{path-specific effects}  \citep{pearl2001direct} 
can help with understanding and addressing complex unfairness scenarios that are relevant to many real-world applications \citep{kilbertus2017avoiding,chiappa2019path,nabi2018fair}. 
We define \emph{path-specific introduced effects} as the difference in some particular path-specific effect on labels and predictions.
Building on this measure, we present some new considerations for how to determine whether paths should be labelled fair. %

The prevalence of introduced total variation is analysed in simulation in Section \ref{sec:empirical}.
Finally, we review related work (\S \ref{related-work}) and discuss findings, limitations, and how these results can be applied (\S \ref{discussion}).

\section{Setup}
\label{sec:setup}

Our fairness analysis focuses on supervised learning algorithms used to make predictions about individuals, specifically regression and classification, and uses structural causal models (SCMs) to represent relationships among variables. Note that Section \ref{sec:psie} relies upon the causal nature of SCMs, whereas
the results of Sections \ref{sec:defining-introduced}, \ref{sec:sep_suf}, \ref{sec:incentives} could also be translated into Bayesian Networks \citep{pearl1986fusion,kollerl2009probabilistic}.

\begin{definition}[Structural causal model (SCM); \citealp{pearl2009causality,pearl2016causal}]%
\label{def:scm}
    A \emph{structural causal model} ${\cal M}$ is a tuple $\langle \exovars, \evars, \structfns, P(\exovars)\rangle$, where $\exovars$ is a set of exogenous (unobserved or latent) variables and $\evars$ is a set of endogenous (observed) variables. $\structfns$ is a set of deterministic functions $\structfns = \{f_{\evar}\}$, where $f_{\evar}$ determines the value of $\evar \in \evars$ based on endogenous variables $\Pa^{\evar} \subseteq \evars \setminus \{V\}$ and exogenous variables  $\exovars^V \subseteq \exovars,$ that is, 
    \ $V \leftarrow f_{\evar}(\Pav{V}, \exovars^V).$
    $P(\exovars)$ is a joint distribution over the exogenous variables, which is assumed to factorise.
\end{definition}

An SCM $\scim$ may be associated with a directed graph ${\cal G}$ which has a node for each variable $B$ %
and an edge $A \to B$ for every $A \in \Pa^B\cup \exovars^B$. Paths from $V_1$ to $V_2$ of arbitrary length are denoted $V_1 \dashrightarrow V_2$. We only consider SCMs for which the graph is acyclic. We refer to ${\cal G}$ as the \emph{associated graph}, %
and say that $\scim$ is \emph{compatible} with $\cid$. %
We often omit the exogenous variables from the graphs.

We define an \emph{SL SCM} to be an SCM containing endogenous variables $Y, \preds, U$,  representing the outcome variable, model prediction, and the loss of some SL model. Specifically, the loss function $f_U$ is real-valued and has two arguments $\target$ and $\preds$ (we consider the
mean squared error $f_U = -(Y - \smash{\hat{Y}})^2$ and zero-one loss $f_U = 0$ if $Y=\smash{\hat{Y}}$; $-1$ otherwise). 
The parents of $\preds$ represent the input features. It may be the case that inputs to $\preds$ are descendants of $\target$. The associated graph is called an \emph{SL graph}. An SL SCM (or graph) includes a \emph{sensitive variable} $A$
if it has an endogenous variable $A$, which represents a sensitive attribute such as\ sex, race, or age.
We assume the possible values for $A$ always include $a_0$ and $a_1$, representing a baseline group and a marginalised group, respectively, though the domain of $A$ may contain more values. 
For example, if $A$ represents racial category, $A$ may take $k\ge2$ values, with $a_0$, $a_1$ representing individuals categorised as white and black respectively.

Following the influence diagram literature \citep{howard1984influence}, we represent $\preds$ with a square node and $U$ with an octagonal node, since $\preds$ can be viewed as a decision optimising the function $f_U$. We also adopt the term \emph{utility variable} to refer to $U$.

An example of SL graph representing a hiring test prediction setting is given in  Figure \ref{fig:zero_one}. 
The training data consists of one input feature $D$, which represents the candidates degree, and a label $Y$, which represents whether the candidate passes or fails.
The graph also include a variable that is not accessible to the predictor, namely a sensitive attribute $A$, which represents gender.
This reflects a scenario in which the sensitive attribute is not available to the developer, or the developer has chosen not to include it as in input, for example due to legal reasons. In this example, all inputs to $\target$ are also inputs to $\preds$, but in general this may not be the case (see later examples).

For an SL SCM $\scim$, we can consider different predictors $\pi: \dom(\Pa^{\hat Y}) \to \dom(\hat Y)$ (where $\dom$ denotes the possible outcomes of a set of variables) by replacing the structural function $f_{\hat Y}$ with $\pi$, which results in a modified SCM $\scim_\pi$.
A predictor $\pi$ is \emph{optimal} if it maximizes the expected value of the utility variable $\mathbb{E}(f_U(\target, \preds))$ given the available features.

\section{Defining Introduced Unfairness}
\label{sec:defining-introduced}

We propose quantifying introduced unfairness with the following approach: (i) select an appropriate measure of unfairness applicable to both $\preds$ and $Y$, and (ii)
calculate the difference in unfairness between $\preds$ and $\target$.
A natural choice of unfairness measure is \textit{total variation}, a generalisation of demographic disparity, which measures the difference in average outcome between different values of the sensitive attribute.~\looseness=-1

\begin{definition}[Average total variation; \citealp{zhang2018fairness}]
\label{def:atv}
The \emph{average total variation} (ATV) on a real-valued variable $V$ is the difference in the expected value of $V$ between the baseline and marginalised group:%
$$\ATV(V) = \mathbb{E}( V \mid A=a_1) - \mathbb{E}( V \mid A=a_0).$$
\end{definition}

We define the new concept \emph{introduced total variation} as the difference in magnitude of ATV between $\pred$ and $Y$.

\begin{definition}[Introduced total variation]
\label{def:itv}
In an SL SCM with real-valued $Y$ and $\haty$, %
the \emph{introduced total variation} (ITV) is:
$$\ITV = |\ATV(\haty)| - |\ATV(Y)|.$$
When ITV is positive/zero/negative we will respectively say that there is \emph{introduced, reproduced}, or \emph{reduced} total variation.
\end{definition}

We illustrate ITV on a hiring test prediction example represented by the SL graph of Figure \ref{fig:zero_one}. %

\paragraph{Example: Hiring test prediction.}
A model predicts job applicants' outcomes on a hiring test using their degree $D$ --- either `maths' or `statistics'. 
Degree is in turn affected by the sensitive attribute gender ($A$).
The loss $U$ depends on the target label $Y \in \{0,1\}$ (representing fail/pass) and on its prediction $\preds$. 
Suppose that 80\% of male applicants have degrees in maths (20\% in statistics), while 20\% of female applicants have degrees in maths (80\% in statistics).
Performance on the test is such that $P(Y=1 \ |\ D=maths) = 51\%$, $P(Y=1 \ |\ D=stats) = 49\%$ (otherwise $Y=0$).
This gives $|\ATV(Y)| = 0.012$.
\begin{figure}[h!]
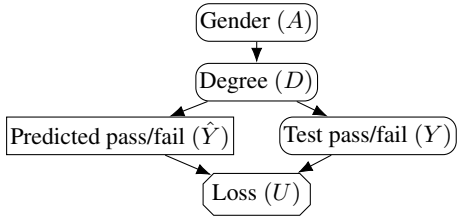

\centering
    \begin{influence-diagram}
       \setcompactsize[node distance=0.4cm]
       \setrectangularnodes
       \node (X2) {Degree $(D)$};
       \node (hatY) [below left = .2 and .3 of X2.south, decision] {Predicted pass/fail $(\hat{Y})$};
       \node (Y) [below right = .2 and .3 of X2.south] {Test pass/fail $(Y)$};
       \node (A) [above = .28 of X2] {Gender $(A)$};
       \node (U) [below = .95 of X2, utility] {Loss $(U)$};
       \edge {A} {X2.north};
       \edge {X2} {hatY};
       \edge {X2} {Y};
       \edge {hatY, Y} {U};
    \end{influence-diagram}
    \caption{SL graph representing hiring test prediction.}
    \label{fig:zero_one}
\end{figure}

\textit{Version 1:}
For an example with $ITV=0$, suppose mean squared error were used, and $\preds$ is a value in $[0,1]$ representing the probability of passing the test. The optimal predictor would be $f_{\preds}(maths)=0.51, f_{\preds}(stats)=0.49$, yielding $|\ATV(\preds)|=0.012$ and therefore $\ITV = 0$.

\textit{Version 2:}
Suppose instead that zero-one loss is used. Then the optimal predictor is $f_{\haty}(maths)=1, f_{\haty}(stats)=0$, yielding $|\ATV(\smash{\haty})|=0.6$, and therefore a high introduced total variation $\ITV=0.588$, due to the fact that 
while female applicants perform only slightly worse than male applicants with respect to $Y$, their predictions $\preds$ are vastly lower.

The existence of models with $\ITV>0$ (including in cases where $\ATV(Y)=0$, as in the later music test example) offers a new perspective on the adage that
\emph{unfair labels lead to unfair models}:
unfair labels may lead to an unfair model,
but sometimes, unfairness may originate exclusively, or predominantly
from other parts of the training process. %

\section{ITV, Separation, and Sufficiency}
\label{sec:sep_suf}

Another way to think about introduced disparities is to ask about reproduced and introduced \emph{dependencies}. These are captured by the existing notions of sufficiency and separation. Indeed ITV is related to the absence of \emph{separation}, a generalization of \emph{equalized odds} to non-binary variables.

\begin{definition}[Separation; \citealp{barocas-hardt-narayanan}]
The random variables $(\preds, \prot, \target)$ satisfy \emph{separation} if \ $\preds$ is independent of $\prot$ conditioned on $\target$, i.e.\
$ \pred \perp \prot \ |\ \target. $
\end{definition}

Absence of separation means that the model has added a \emph{dependence} 
between $A$ and $\preds$, that was not present between $A$ and $\target$. In contrast, ITV asks whether the model has added to the \emph{disparity} in $\preds$ (as measured by total variation), compared to that in $\target$. Thus, they both detect whether some effect has been introduced by the model. While lack of separation indicates that the model has introduced some new dependence, ITV captures whether this manifests as an increase (or decrease) in variation between groups. That is, ITV measures the group level impact resulting from the introduction of some new dependency by the model.

For example, in the test prediction examples (Figure \ref{fig:zero_one}), separation is not satisfied in either version. But in version 2 (where all statisticians are rejected), we see a large introduced variation. In version 1 (where statisticians are given slightly lower scores), we have $ITV=0$.

In the binary case, separation prevents introduced variation, as established next by Proposition \ref{lem:sep}.
The converse is not true: it is possible for a model to lack separation while there is a reduced or reproduced variation. For example, version 1 of Figure 1 lacks separation and has $\ITV=0$.

\begin{restatable}[]{proposition}{thmsepatv}
\label{lem:sep}
Let $\dom(Y) = \{0,1\}, \dom(\smash{\hat{Y}}) \subseteq [0,1]$, and $\dom(A)\supseteq \{a_0,a_1\}$, where $a_0, a_1$ are the baseline and marginalised groups.
Then separation implies $\ITV \le 0$, i.e. there is not introduced total variation.
\end{restatable}
\begin{proof}
$ | \ATV(\haty) |  = |  \mathbb{E}(\haty | a_1 ) -  \mathbb{E}(\haty | a_0 )  |$
\begingroup
\allowdisplaybreaks
\begin{align*}
      & =  |   \textstyle \sum_y P(y|a_1)\mathbb{E}[\haty | a_1, y] - \sum_y P(y|a_0)\mathbb{E}[\haty | a_0, y]  |  \\
      & = | \textstyle \sum_y P(y|a_1)\mathbb{E}[\haty | a_1, y] - \sum_y P(y|a_0)\mathbb{E}[\haty | a_1, y] | \\
      & \hspace{170pt} (\text{by separation})\\
      & = | \big( P(Y=1|a_1) - P(Y=1|a_0) \big) \mathbb{E}[\haty | a_1, Y=1]  \\
      &   \quad - \big( P(Y=1|a_1) - P(Y=1|a_0)) \big) \mathbb{E}[\haty | a_1, Y=0]  | \\
      & \hspace{180pt} (\text{$Y$ is binary})  \\  
      & = | \big( P(Y=1|a_1) - P(Y=1|a_0) \big) &\\
      & \quad . \big( \mathbb{E}[\haty | a_1, Y=1] - \mathbb{E}[\haty | a_1, Y=0] \big) | \hspace{35pt}  (\text{factor})  \\
      & \le | P(Y=1 \mid a_1) - P(Y=1 \mid a_0)| \hspace{18pt} (\text{as }  0 \le \haty \le 1) \\
      & =| \ATV(Y)| .     \qedhere  
                   \end{align*}
                   \endgroup
\end{proof}

We also establish the relationship between \textit{sufficiency} and ITV. Sufficiency generalises the notion of \textit{predictive parity}, and is closely related to the notion of \textit{calibration by group} \citep{barocas-hardt-narayanan}. Sufficiency means that the predictor $\preds$ fully captures the dependencies between $\prot$ and $\target$ (but does not prohibit additional dependencies being introduced by the model).

\begin{definition}[Sufficiency; \citealp{barocas-hardt-narayanan}]
The random variables $(\preds, \prot, \target)$ satisfy \emph{sufficiency} if \ $\target$ is independent of $\prot$ conditioned on $\preds$, i.e.\
$ \target \perp \prot \ |\ \pred. $
\end{definition}

If sufficiency holds, a model may still introduce additional variation. In fact, sufficiency prevents \emph{reduced} variation in the binary case:

\begin{restatable}[]{proposition}{thmsufatv}
\label{lem:suf}
Let $\dom(\preds) = \{0,1\}, \dom(Y) \subseteq [0,1]$, and $\dom(A)\supseteq \{a_0,a_1\}$, where $a_0, a_1$ are the baseline and marginalised groups.
Then sufficiency implies $\ITV \ge 0$, i.e.\ there is not reduced total variation.
\end{restatable}

\begin{proof}
Swap $Y$ and $\hat{Y}$ in the proof of Proposition \ref{lem:sep}.
\end{proof}

In Appendix \ref{sec:app_info_theory}, we consider a related measure, \emph{introduced mutual information}, which can be applied to cases where $\preds$ and $\target$ are categorical or continuous. Analogous results to Propositions \ref{lem:sep} and \ref{lem:suf} hold in this more general setting.
A corollary of these results is that it is often impossible to even approximately satisfy sufficiency and independence requirements simultaneously.

\section{Incentives for ITV}
\label{sec:incentives}

Under what circumstances will introduced variation arise? 
Since an arbitrary predictor can introduce variation in almost any setting, we focus on predictors that have been trained to optimality in their given setup. In other words, we ask when introduced variation is \emph{incentivised}.

To specify the graphical criteria, we use the well-known concept of d-separation, which identifies conditional independencies based on the paths between variables.

\begin{definition}[d-separation; \citealp{Verma1988soundness}]%
\label{def:d-separation}
A \emph{path} $V_1 \upathto V_k$ is a sequence of distinct nodes $V_1,...,V_k, k\geq 0$ such that 
every pair of consecutive nodes is connected by an edge
$V_i \to V_{i+1}$ or $V_i \leftarrow V_{i+1}$.
When three consecutive nodes in a path have converging edges $V_{i-1} \to V_i \gets V_{i+1}$, we call $V_i$ a \emph{collider}.
    A path $p$ is said to be \emph{blocked} by the conditioning set $\sZ \subseteq \evars$ if $p$ contains a 
    non-collider $W$ in $\sZ$
    or a collider $W$ that is neither equal to, nor an ancestor of, any $Z \in \sZ$.
For disjoint sets $\sX$, $\sY$, $\sZ$, the conditioning set $\sZ$ is said to \emph{d-separate} $\sX$ from $\sY$,
if and only if $\sZ$ blocks every path
from a node in $\sX$ to a node in $\sY$. Sets that are not d-separated are
called \emph{d-connected}.~\looseness=-1
\end{definition}

If $X$ and $Y$ are d-separated by $Z$, then $X$ and $Y$ are conditionally independent given $Z$ in any SCM compatible with the graph, i.e.\ $P(X\mid Y, Z) = P(X\mid Z)$ \citep{Verma1988soundness}.
A consequence of d-separation of particular value to us, is that it can be used to establish which features can be useful to an optimal predictor.

\begin{definition}[Requisite feature; \citealp{Lauritzen2001}]
In an SL graph, a feature $W \in \Payhat$ 
is \emph{requisite} if it is d-connected to $U$
conditional on $\Payhat \cup \{\haty\} \setminus \{ W \}$.
Let 
$req(\Payhat)$
denote the set of requisite features.
\end{definition}
\begin{lemma}[\citealp{Fagiuoli1998,Shachter2016}]
\label{le:requisite}
Every SL SCM has an optimal predictor $\pi$ that only depends on requisite features, i.e.\ $P(\hat Y \mid \Payhat) =
P(\hat Y \mid req(\Payhat))$.
\end{lemma}

\subsection{Arbitrary loss functions}

We begin with a graphical criterion for when ITV may be incentivised under arbitrary loss functions.

\begin{restatable}[Introduced total variation criterion]{theorem}{theoremitvgc}
\label{theorem:itv_gc}
An SL graph $\mathcal{G}$ is compatible with an SCM $\scim$ in which all optimal predictors have $\ITV>0$ iff there is a requisite feature $W \in \Payhat$ that is d-connected to $A$.
\end{restatable}

\begin{proof} %

We first show that the criterion is sound.
Suppose that $A$ is not d-connected to any requisite feature.
By \cref{le:requisite}, there exists an optimal predictor $\pi$ that only responds to requisite features.
Then 
\begin{align*}
    \mathbb{E}_\pi(\smash{\hat{Y}} \ |\ a) 
    &= \EE(\EE(\pred\mid \Payhat, a)\mid a) 
    & (\text{total expectation})\\
    &= \EE(\EE(\pred\mid \Payhat)\mid a) 
    & (\text{since $\pred\dsep A\mid \Payhat$})\\
    &= \EE(\EE(\pred\mid req(\Payhat))\mid a) 
    & (\text{by \cref{le:requisite}})
    \\
    &= \EE(\EE(\pred\mid req (\Payhat) )) 
    & (\text{by assumption})
    \\
    &= \EE(\pred).
    & (\text{total expectation})
\end{align*}
Therefore $\ATV(\smash{\hat{Y}})=0$.
It follows that $\ITV \le 0$, i.e. there is no introduced total variation.

For the completeness direction of the proof, in Appendix \ref{sec:app_itv_proofs} we construct a compatible SCM with $\ITV > 0$ for any SL graph in which $A$ is d-connected to a requisite feature.
\end{proof}

For example, in Figure \ref{fig:zero_one},  $D \in \Payhat$ is d-connected to $A$ and is requisite, and thus the graph satisfies the ITV criterion. This graph is therefore compatible with an ITV incentive, as verified for the particular model stated in version 2, where $\ITV=0.588 > 0$ for the only optimal predictor.

The ITV criterion can be broken down into two conditions, each with an easily interpreted meaning.
The first condition says that it is only possible for a predictor to introduce total variation if some feature $W$ can statistically depend on the sensitive attribute $A$.
Otherwise, the total variation of $\preds$ will be zero (even if the labels $Y$ are strongly dependent on $A$), and so ITV cannot be positive.
The second condition says that ITV can only be incentivised if such a feature $W$ is important for optimal predictions.
Indeed, if $A$ is only connected to features that are unimportant for predicting $Y$, then an optimal predictor may avoid any dependency with $A$.
Note that these conditions can be stated purely in terms of conditional independencies rather than d-separation, so the ``only if'' part of the theorem can be adapted to settings where we know the joint distribution rather than the graph.

While it is relatively easy to see that both conditions are necessary, the theorem also establishes the converse: that jointly satisfying the conditions is sufficient for an introduced total variation incentive under some model compatible with the graph.
This latter \emph{completeness direction} of the proof is related to the corresponding completeness proof for d-separation \citep{geiger1990d}. However, our result is not a corollary of the d-separation result. In particular, the completeness results of d-separation rely on being able to freely specify conditional probability distributions for all nodes. This is not possible here since we are concerned with \textit{optimal} predictors, and so the distribution at $\preds$ cannot be independently selected \citep{Everitt2021agent}.

\Cref{theorem:itv_gc} thus gives some insight to the question posed by the title of this paper: predictions can be unfair in spite of fair labels, because an optimal predictor may need to depend on some feature that is correlated with $A$.
The fact that the conditions of the ITV criterion theorem are easily satisfied indicates that introduced unfairness is possible in a wide range of scenarios.

\subsection{P-admissible loss functions}
\label{sec:padm}
Ideally, we would not just quantify disparities introduced by a system, but would understand what components of the system may be controlled to reduce them.
One such component is the training loss function.

As a simple example, if zero-one loss is used, this can  lead to a large ITV, because small group differences can be amplified into large differences in ``all or nothing'' predictions (recall version 2 of the hiring test prediction example). %
Can this amplification be prevented by choosing a ``better behaved'' loss function? We investigate an existing class of loss functions that incentivise the predictor to output the expected value of $Y$ given the system inputs.

\begin{definition}[P-admissible; \citealp{miller1993loss}]
For an SL SCM $\mathcal{M}$ with utility variable $U$, we say that $f_U$ is a \emph{P-admissible loss function} if $\pi(\Payhat) := \mathbb{E}(Y \mid \Payhat) $ is an optimal predictor.
\end{definition}

Examples of P-admissible loss functions include mean squared error and cross-entropy loss \cite{miller1993loss}.
For some graphs, using a P-admissible loss function rules out the possibility of an ITV incentive.

\begin{restatable}[P-admissible ITV criterion]{theorem}{theorempadmisgc}
\label{thm:padmissible_gc}
An SL graph $\mathcal{G}$ is compatible with an SCM $\mathcal{M}$ for which $f_U$ is P-admissible and all optimal predictors have $\ITV>0$ only if in addition to the conditions of Theorem \ref{theorem:itv_gc}, $A \notin \Payhat$ and $A$ is d-connected to $U$ conditioned on $\Payhat$.
\end{restatable}

\begin{proof} %
If the \cref{theorem:itv_gc} conditions do not hold, then we already know that $\ITV=0$ for at least one optimal policy under any loss function, including P-admissible ones.
Consider therefore an SL graph $\G$ in which the extra graphical condition of \cref{thm:padmissible_gc} does \emph{not} hold.
Then for any $a \in \dom(A)$:
\begin{align*}
    \mathbb{E}( \haty \ | \ a ) 
      & = \mathbb{E}( \mathbb{E}( Y \mid \Payhat) \ | \ a )
      &(\text{by P-admissibility})\\
      & = \mathbb{E}( \mathbb{E}( Y \mid \Payhat, a) \ | \ a )                        &(\text{see below})\\
      & = \mathbb{E}( Y \ | \ a ) .
      &(\text{law of total probability})
\end{align*}
The second equality holds if either
(a) $A \in \Payhat$, or
(b) $A$ is d-separated to $Y$ conditioned on $\Payhat$.
Since we have assumed that the the extra condition of \cref{thm:padmissible_gc} does not hold, either (a) or (b) must be true.
Thus the second equality follows, and we have established that
$\mathbb{E}( \haty \ | \ a ) = \mathbb{E}( Y \ | \ a ) $ for all $a$.
From this it follows that $\ITV=0.$

This establishes that if the graphical conditions of \cref{thm:padmissible_gc} do not hold and the loss is P-admissible, then $\ITV=0$ for an optimal policy.
\end{proof}

Compared to \cref{theorem:itv_gc}, \cref{thm:padmissible_gc} adds the extra condition that $A$ has to be a non-parent of $\preds$ and d-connected to $U$.
This is the graphical condition for $A$ to provide additional \emph{value of information} \citep[Thm.~9]{Everitt2021agent}.
That is, a predictor with access to $A$ can have lower loss than one without access.
Since adding $A$ as an observation prevents ITV, knowing $A$ must help the predictor 
disentangle information about $A$ and $Y$.
In other words, predictors with P-admissible loss can become unfair in spite of fair labels because they are unable to disentangle information about $A$ and $Y$.

For example, Figure \ref{fig:zero_one}
does not satisfy the extra condition of Theorem \ref{thm:padmissible_gc} as the only paths from $A$ to $Y$ are blocked by $D$. Conceptually, the predictor therefore lacks incentive to infer $A$. Indeed, if a P-admissible loss function is used, the optimal policy becomes $f_{\preds}(maths)=0.51, f_{\preds}(stats)=0.49$ which gives $\ITV = 0$, as we saw in version 1.
\looseness=-1

We now present a case that satisfies the conditions of Theorem \ref{thm:padmissible_gc}, and allows for $\ITV>0$ even under P-admissible loss.

\paragraph{Example: Music test prediction.}

\begin{figure}[h!]
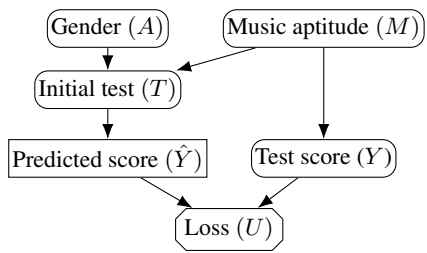

    \centering
\begin{influence-diagram}
   \setcompactsize[node distance=0.4cm]
   \setrectangularnodes
   \node (middle) [phantom] {};
   \node (hatY) [below left = of middle, decision] {Predicted score $(\hat{Y})$};
   \node (Y) [below right = of middle] {Test score ($Y$)};
   \node (X1) [above = of hatY] {Initial test $(T)$};
   \node (A) [above = of X1, yshift=-1mm] {Gender $(A)$};
   \node (X2) at (Y|-A) {Music aptitude $(M)$};
   \node (U) [below = 1.2 of middle, utility] {Loss $(U)$};
   \edge {A} {X1};
   \edge {X1} {hatY};
   \edge {X2} {X1};
   \edge {X2} {Y};
   \edge {hatY, Y} {U};
\end{influence-diagram}
\caption{SL graph representing music test prediction.} \label{fig:music}
\end{figure}

Consider the SL graph in Figure \ref{fig:music}, which  represents a music test scenario in which a model is trained to predict the outcome $Y\in \{0,1\}$ of a test taken at the end of a music course (adapted from \citealp{chiappa2019causal}). The prediction is based only on an initial test outcome $T\in \{0,1\}$, which has a gender bias.
Assume equal numbers of females and males, both with equal numbers of low and high musical aptitude (represented by $M=0,1$ respectively), take tests $T$ and $Y$. Suppose that 95\% of individuals with high aptitude $(M=1)$ pass the final test ($Y=1$), compared to 5\% of individuals with low aptitude $(M=0)$. Suppose that 90\% of high aptitude females pass the \emph{initial} test ($T=1$), compared to 100\% of males. Low aptitude individuals also pass test $T$ with 5\% probability.

In this scenario, the predictor has an incentive to learn $A$, since learning $A$ would enable it to better understand what the biased test $T$ says about true aptitude $M$.
Formally, the conditions of Theorem \ref{thm:padmissible_gc} are satisfied, since $A$ is d-connected to the requisite feature $T$, and conditioning on $T$ opens the path from $A\notin \Payhat$ to $U$ (via $T$, $M$, $Y$).
As expected, we find that the optimal predictor does have an introduced total variation under both zero-one loss
and under mean squared error %
(see \cref{tab:music-predictors}).

\begin{table}[]
    \centering
    \begin{tabular}{|c|c|c|c|c|}
    \hline
        $T$ & $A$ & $\pred$ (0-1) & $\pred$ (P-adm) & $\pred$ (P-adm+$A$ feature)\\\hline
        0   & 0   & \multirow{2}{*}{0}& \multirow{2}{*}{0.1}& 0.01 \\\cline{1-2}\cline{5-5}
        0   & 1   & & & 0.14 \\\hline
        1   & 0   & \multirow{2}{*}{1} &\multirow{2}{*}{0.905} & 0.907  \\\cline{1-2}\cline{5-5}
        1   & 1   & & & 0.903 \\\hhline{|=|=|=|=|=|}
        \multicolumn{2}{|c|}{}      & 0-1  & P-adm & P-adm+$A$ feature\\
        \hline
        \multicolumn{2}{|c|}{ITV} & 0.05 & 0.04 & 0\\
         \hline
    \end{tabular}
    \caption{Impact of losses and features on the music example. The final column gives the values attained if $A$ was also added as a feature.}
    \label{tab:music-predictors}
\end{table}

This shows that an introduced variation can arise even when the training labels are completely unbiased, i.e.\ when $A$ is independent of $Y$ and $\ATV(Y)=0$. \citet{wang2019balanced} describe a similar dynamic in the context of image classification.

\paragraph{Removing an ITV incentive.}
The good news is that preventing the predictor from trying to infer $A$ is often as simple as providing $A$ as a feature to the predictor.
We state this important insight as a corollary of \cref{thm:padmissible_gc}.

\begin{corollary}
\label{cor:p-adm}
If the sensitive variable $A$ is available as a feature to the predictor,
then $\mathcal{G}$ is not compatible with an ITV incentive 
under P-admissible loss functions.
\end{corollary}

This result generalises the observation by \citet{chiappa2019causal} (revisited from \citealp{kusner2017counterfactual}) that a linear least-squares predictor with access to the sensitive variable $A$ will ``strip off'' the bias in $T$ 
from $\hat{Y}$ when the data generating 
process consists of linear relationships. 
This challenges the notion of ``fairness through unawareness'', as it suggests that making the sensitive attribute available as a feature can improve fairness when labels are fair.
\Cref{cor:p-adm} %
reveals that a similar dynamic holds  %
even when the data generating process is nonlinear.
Indeed, adding $A$ as feature in our music example results in $\ITV = 0$ under the optimal policy described in \cref{tab:music-predictors}.

As with any technique to ensure fairness, making $A$ available as a feature should not be done without an understanding of the context.
In particular, since $\ITV=0$ is a specific group-level measure, it does not come with individual-level guarantees.
In the music test example, as the initial test has lower accuracy for women, women who pass the initial test receive a slightly lower prediction when $A$ is used explicitly compared to when it is not (0.903 instead of 0.905, see \cref{tab:music-predictors}). 
Even though this negative effect is offset by the higher score given to women who failed the test (0.14 instead of 0.1),
this may still be perceived as unfair by the high aptitude women who passed the test $T$.

We note that in the case of classification, requiring a discrete deterministic prediction will mean that a P-admissible loss function cannot be used. For instance, if mean squared error is used to produce $\smash{\hat{Y}}=p \in [0,1],$ but a binary accept/reject is required, then thresholding (e.g.\ at 0.5) reduces to the zero-one loss case, and may give $\ITV> 0$, even if the Theorem \ref{thm:padmissible_gc} criteria are met. In this case, randomising the result (accepting with probability $p$) preserves the result.
However, our results do not rely on randomness in general. E.g.\ consider situations where the prediction task is inherently continuous, such as a sum of
money paid out to an insurance customer. Then there would
be no need to randomise (or threshold) and the results would
be preserved.

\section{Path-specific Introduced Effects}
 \label{sec:psie}

In previous sections, we have examined introduced unfairness for statistical definitions of unfairness.
However, causal definitions can offer a richer understanding of unfairness \citep{pearl2009causality,kusner2017counterfactual,zhang2017causal,loftus2018causal,chiappa2019path,chiappa2020general,oneto2020fairness}. 
In this section, we consider a notion of introduced unfairness based on causal effects restricted to certain paths, referred to as \emph{path-specific effects} (PSEs). We first recap causal interventions and path-specific effects, and then adapt the idea of introduced unfairness to define \emph{introduced path-specific effects}, and illustrate how they may be used to examine the source of an introduced effect.

\subsection{Background on path-specific effects}
As well as allowing us to investigate the result of conditioning on a particular variable, SCMs also allow us to investigate the result of \emph{intervening} on a particular variable, to answer causal questions.
Formally, an intervention in an SCM $\scim$ consists in setting a variable $X$ to the value $x$ by replacing the structural function $f_X$ with a constant function $f_X=x$.
The variables in the modified model are referred to as $V_x$.
Interventions on $X$ only alter the values of variables descending from $X$, so $V_x=V$ for non-descendants of $X$. \emph{Path-specific} interventions are a more targeted type of intervention, that only propagate along specific paths. %
While global interventions allow us to reason about the total causal effect of a variable, for example to answer the question, ``What effect did being male have on being hired in a job application?'', path-specific interventions enable us to reason about the effect along a subset of paths, for example to answer a more fine-grained question, ``What effect did indicating male associated hobbies on resumes have on being hired in a job application?''.

\begin{definition}[Path-specific effect; \citealp{pearl2001direct}]
For a given edge-subgraph $\mathcal{P}$ specifying  a set of paths in an SCM $\scim$, let $\mathcal{M}_{\mathcal{P}}$ be a modified version of $\mathcal{M}$ in which all function inputs not in $\mathcal{P}$ are kept fixed at a baseline value $A=a_0$. That is, replace each structural function $f_X(V^1, \dots , V^k)$ in $\mathcal{M}$
with the function $(f_X)'$, equal to $f_X$, except that if $V^i \rightarrow X$ is not in $\mathcal{P}$, then when evaluating at $\exovals$ the argument $V^i$ is replaced with the constant $V^i_{a_0}(\exovals)$.
The \emph{path-specific response} $V_{\mathcal{P}(a_0 \rightarrow a_1)}$ is defined as $V_{a_1}$ in the model $\mathcal{M}_{\mathcal{P}}$.
The  \emph{path-specific effect} (\PSE) on a real-valued variable $V$ is:
$$ \PSE(V) = \mathbb{E}( V_{\mathcal{P}(a_0 \rightarrow a_1)}) - \mathbb{E}( V_{a_0}) . $$
\end{definition}

\subsection{Auditing ML system outputs for fairness -- a risk when labelling paths to $\preds$ as fair/unfair}

\noindent Path-specific effects can be used to inform judgements about whether a decision policy is unfair. For example, in the case of Berkeley’s alleged sex bias
in graduate admissions, the original analysis considered direct effects ($\textit{Gender} \rightarrow \textit{Outcome}$) to be unfair, but indirect effects via ($\textit{Gender} \rightarrow \textit{Department} \rightarrow \textit{Outcome}$) to be fair \cite{bickel1975sex,pearl2009causality}. This approach assumes that societal considerations can be used to label paths between the sensitive variable $A$ and the outcome as fair (or ``justified'') or unfair, and outcomes are declared unfair if (significant) effects are found along any unfair path.

Suppose instead that the aim is to audit the fairness of a trained machine
learning (ML) system, by investigating the system outputs. While understanding which paths are responsible for disparate $\preds$ is crucial, here we show that the training process must also be taken into account before attempting to label paths from a sensitive variable $A$ to $\smash{\hat{Y}}$ as fair or unfair.

\paragraph{Example: Penalising female dominated degrees.}

\begin{figure}[hb!]
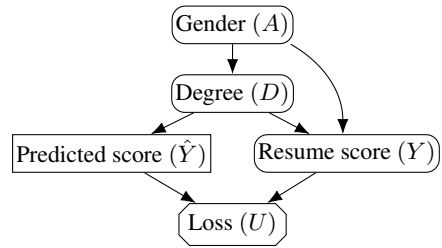

    \centering
    \begin{influence-diagram}
       \setcompactsize[node distance=0.4cm]
       \setrectangularnodes
       \node (X2) {Degree $(D)$};
       \node (hatY) [below left = of X2.south, decision] {Predicted score $(\hat{Y})$};
       \node (Y) [below right = of X2.south] {Resume score $(Y)$};
       \node (A) [above = of X2] {Gender $(A)$};
       \node (U) [below = 1.2 of X2, utility] {Loss $(U)$};
       \edge {A} {X2.north};
       \edge {X2} {hatY};
       \edge {X2} {Y};
       \edge {hatY, Y} {U};
       \path (A) edge[->, bend left] (Y);
    \end{influence-diagram}
    \caption{Penalising female degrees example. }
    \label{fig:biased_recruiter_B}
\end{figure}

In the SL graph of Figure \ref{fig:biased_recruiter_B}, a system $\preds$ is used to score resumes based on the applicant's degree $D$. The system is trying to emulate scores $Y$ given by humans, that also directly depend on applicant's gender $A$.

If we deem degree to be a reasonable decision criterion for the job in question, it is tempting to label the path $A\to D \to \preds$ as fair, and therefore to conclude that the decision given by $\preds$ must be fair.

Assume for simplicity that there are only two degrees, maths and statistics, and that they are equally valuable, with reviewers giving them both a score of 5. In addition, suppose that male and female applicants are (unfairly) given an additional score of $1$ and $-1$ respectively. Moreover, suppose that 80\% of mathematicians are male, while 80\% of statisticians are female, and that mean squared error is used as the loss function. Then the optimal predictor is
$f_{\preds}(maths) = \mathbb{E}( Y \ |\ D=maths) = 5 + 0.6 = 5.6$ and $f_{\preds}(stats) = \mathbb{E}( Y \ |\ D=stats) = 5 - 0.6 = 4.4.$~\looseness=-1

Thus statistics is given a lower score by the model, even though it is just as valuable as mathematics, \emph{because it has a higher percentage of women}. Knowing this, we may instead conclude that these decisions at $\preds$ are \textit{not} fair, because the path-specific effect of $A$ on $\hat Y$ via $D$ is stronger than the corresponding path-specific effect on $Y$.
To formalise this type of situation, we define path-specific introduced effects.

\begin{definition}[Path-specific introduced effect]
\label{def:psie}
Let $\scim$ be an SL SCM with real-valued variables $Y$ and $\hat{Y}$. Let $\mathcal{P}$ denote an edge-subgraph with paths from $A$ to $Y$ and $A$ to $\smash{\haty}$.
Then the \emph{path-specific introduced effect} (PSIE) is defined as\looseness=-1
\begin{equation*}
\PSIE_\mathcal{P}  
=
|\PSE_{\mathcal{P}}(\haty)| - |\PSE_{\mathcal{P}}(Y)|. \end{equation*} 
In particular, if $\mathcal{P}$ consist of all directed paths of the form $A \dashrightarrow X \dashrightarrow \preds$ and $A \dashrightarrow X \dashrightarrow Y$, then
$\PSIE_\mathcal{P}$ is the \emph{path-specific introduced effect via variable $X$}.
\end{definition}

For example, in Figure \ref{fig:biased_recruiter_B}, let $\mathcal{P} := \{ A\rightarrow D$\linebreak[1] $\rightarrow \smash{\hat{Y}}$,
$A\rightarrow D \rightarrow Y$\}. If we compute $\PSIE_\mathcal{P}$ with man as the baseline group for $A$, we find that $\PSIE_\mathcal{P} = 0.72 - 0 = 0.72$ and conclude that there is a PSIE via $D$. 
In other words, $D$ is carrying a (problematic) amplified effect, as a result of the path partially reproducing the spurious effect along the path \mbox{D $\leftarrow A \rightarrow Y$}.
As the effect along $A \to D \to Y$ may be considered fair in this example, we may say that the effect of $A$ on $\hat Y$ via $D$ is greater than that justified by the corresponding path to $Y$.
Note that looking at the introduced total variation would not identify this phenomena, as $\ITV=-1.28<0$ in this example.

In this example, the training labels $Y$ are unfair. However, a PSIE via a variable $X$ can also occur in cases where $Y$ is considered fair. For example, assume that, in Figure \ref{fig:biased_recruiter_B}, $A \rightarrow Y$ is replaced with $A \rightarrow \textit{Coding Experience} \rightarrow Y$ and that coding is deemed fair (despite more men having coding experience) since it is relevant for the job. Then $D$ can nonetheless still carry a problematic PSIE exactly as before -- statistics will be still be penalised for having a higher percentage of women.

These examples highlight the importance of taking PSIE into account when deciding whether an effect along a path is fair or unfair.

\paragraph{Relationship to proxy unfairness.}

\citet{kilbertus2017avoiding} define proxy discrimination as arising if a causal path from $A$ to a decision is blocked by a variable deemed to be a proxy (e.g.\ someone's name may be considered a proxy for their gender), but does not describe how to ascertain whether a variable should be considered a proxy. PSIE gives information that can help judge whether a variable is a (problematic) proxy, namely whether it carries an amplified effect from $A$. In the examples described above, the seemingly harmless degree carries an amplified effect; it thus acts as a harmful proxy for gender.

\subsection{Enforcing ML system outputs to be fair - A risk when reducing unfair PSEs.}

Several approaches in the literature are based on enforcing path-specific effects or counterfactual extensions that are considered problematic in the data 
not to be transferred to the system (e.g.\ \citet{nabi2018fair,chiappa2019path}).
These approaches implicitly assume that the prediction model and training data share the same underlying causal structure, and ensure that the effect on any path corresponding to an unfair path underlying the data is reduced, either by constraining the objective during training \citep{nabi2018fair} or by performing a path-specific counterfactual prediction at test time \citep{chiappa2019path}. However, the discussion above indicates that effects on paths that are deemed fair also need to be considered. Specifically, consider Figure \ref{fig:biased_recruiter_B}, but with an additional direct path $A \to \preds$, so that the causal structure underlying $Y$ and $\preds$ are the same. Ensuring that the effect along the harmful path $A \rightarrow Y$ is not reproduced as $A \rightarrow \preds$ is not sufficient to ensure fairness: the effect via $A \rightarrow D \rightarrow \preds$ needs to also be understood. Any method that constrains the learning to reduce the effect along ``unfair'' paths risks transferring this effect to a ``fair'' path, such as the one through $D$. Methods that only learn the causal model underlying the data without such constraints might still carry some risk.
PSIE can be used to formalise these risks. In addition, our formalism can be used to understand when such an amplified effect (naturally) arises as a consequence of optimality, i.e.\ when such an effect is incentivised.

\subsection{Incentives for PSIE}
As we did for ITV in the previous section, here we ask: when may PSIE be incentivised by a training setup?
The conditions are very similar to Theorem \ref{theorem:itv_gc}.

\begin{restatable}[PSIE graphical criterion]{theorem}{theorempsiegc}
\label{theorem:psie_gc}
An SL graph $\mathcal{G}$ is compatible with an SCM $\scim$ in which all optimal predictors have $\PSIE_{\mathcal{P}} > 0$ iff there is some path $p \in \mathcal{P}$ of the form $A \dashrightarrow W \rightarrow \yhat$ where $W\in req(\Payhat)$ is a requisite feature.
\end{restatable}

\begin{proof} %
Let $\mathcal{G}$ be an SL graph $\G$ and $\mathcal{P}$ an edge-subgraph that includes no path $A\pathto W \to \pred$ via a requisite feature $W$.
By \cref{le:requisite}, for any compatible SCM there is an optimal predictor that only depends on requisite features.
Under this predictor $\PSE_{\mathcal{P}}(\pred)=0$.
Since $\PSIE_{\mathcal{P}}\leq |\PSE_{\mathcal{P}}(\pred)|$, we have established that the graphical criterion is sound.
A proof of completeness can be found in Appendix \ref{sec:app_psie_proofs}.
\end{proof}

For certain edge-subgraphs $\mathcal{P}$, the conditions for PSIE under P-admissible loss are identical to \cref{theorem:psie_gc}: for example when $\mathcal{P}$ contains only paths to $\pred$ and none to $Y$.
Whether less degenerate choices of $\mathcal{P}$ yield a stronger condition for P-admissible loss functions is an open question.

\section{Empirical Results}
\label{sec:empirical}
Our graphical criteria give conditions under which a graph is \emph{compatible} with ITV, or PSIE. But do these arise in practice?
For random distributions, d-connectedness almost always implies conditional dependence \citep{Meek1995}. Therefore the requisite feature $W$ required by \cref{theorem:itv_gc,thm:padmissible_gc,theorem:psie_gc}
will almost always have a dependency with $A$ and $U$ when the criteria are satisfied (under random distributions). However, this does not necessarily imply positive ITV. In particular, our completeness results establish the existence of some model where ITV is positive, but not how common positive ITV is. We address this second question empirically.

PyCID is an open source Python library for graphical models of decision-making \citep{fox2021pycid}. Using PyCID's method for generating random graphs, we sample SL SCMs with 6 nodes that satisfy the graphical criteria of \cref{theorem:itv_gc,thm:padmissible_gc}, and used a Dirichlet distribution to assign random distributions to each node.
Out of 1000 samples for each, we found that $20\%$ of the models satisfying the graphical criteria of \cref{theorem:itv_gc} have ITV greater than 0.01 under zero-one loss, while $16\%$ of the models satisfying the graphical criteria of \cref{thm:padmissible_gc} have ITV greater than 0.01 under P-admissible loss.
The results did not appear particularly sensitive to variations in the number of nodes or the edge density of the random graph.
The results can be reproduced via the linked colab%
\footnote{\url{https://github.com/causalincentives/pycid/blob/master/notebooks/Why_fair_labels_can_yield_unfair_predictions_AAAI_22.ipynb}}, which also shows how the examples discussed above can be analysed using PyCID.

\section{Related Work}
\label{related-work}

\paragraph{Statistical approaches.}
In Section \ref{sec:defining-introduced}, we discuss the relationship between ITV and separation and sufficiency, for example that lack of separation means that a dependency has been introduced (in the binary case), whereas ITV measures %
an increased (or decreased) disparity.

\paragraph{Causal approaches.}
The ability to account for the complex patterns that underlie the data generation process makes causal models a powerful tool for reasoning about fairness.
As such, causal models are increasingly used both for measuring and alleviating unfairness in ML systems \citep{chiappa2019causal,creager2020causal,loftus2018causal,nabi2019learning,plecko2020fair,qureshi16causal,russell2017worlds,zhang2018fairness,zhang2017causal}. The idea of inferring the presence of unfairness in data with path-specific effects and counterfactuals dates back to \citet{pearl2009causality} and \citet{pearl2016causal}. \citet{kilbertus2017avoiding,kusner2017counterfactual} and \citet{nabi2018fair} develop approaches for training ML systems that achieve a coarse-grained version of path-specific fairness, counterfactual fairness, and path-specific fairness respectively. The following work of \citet{chiappa2019path} and \citet{chiappa2020general} introduces general methods for achieving path-specific counterfactual fairness, while \citet{wu2019pc} discuss identification issues, and how to compute path-specific counterfactuals.~\looseness=-1

Attempts to describe the relation between the data and model outputs $Y$ and $\smash{\hat Y}$ have appeared in some of these works, with the goal of elucidating limitations of statistical fairness definitions at a high level \citep{chiappa2019causal,kilbertus2017avoiding}. \citet{zhang2018equality} is the first work to more thoroughly characterise the causal connection between the two variables, by linking %
the equalised odds criterion to the underlying causal mechanisms. 
This work differs from ours in several ways. Our goal in characterising the relation between $Y$ and $\smash{\hat Y}$ is not to connect statistical fairness definitions to the underlying data generation mechanisms, but to formalise the notion that models may introduced or amplify causal effects not present in the training labels. In addition, rather than reasoning about a trained model for $\smash{\hat Y}$, we also incorporate the training mechanism by considering the necessary behaviour of optimal predictors \citep{Everitt2021agent}. This enables us to %
characterise when policies are incentivised to introduce or amplify disparities that were not present in the training labels. %

\paragraph{Amplified disparity in context.}
There is also a broader literature that investigates the relationship between biased labels and biased models for particular applications, such as object recognition \citep{wang2019balanced, zhao2017men}. For example, this may result from the fact that even if $Y$ and $A$ are independent,
some features $X$ might be correlated with both $A$ and $Y$,
inducing a correlation between $A$ and $\preds$ \citep{wang2019balanced}.
This can be seen as an example of ITV.
In contrast to these works, we seek a theoretical understanding
of when introduced disparity will arise, particularly in
decision-making settings about individuals.

\section{Discussion}
\label{discussion}

\paragraph{Applicability of incentive criteria.} The graphical criteria can be used to analyse the potential incentives of a system that is yet to be built, or for which we lack access to model outputs for other reasons (e.g.\ a proprietary system). 
The necessary graphical knowledge may come from domain expertise, previous studies, or data.
For example, a developer or auditor may know that \textit{Ability} is a joint ancestor of \textit{Test score} and \textit{Job performance}, even if they are unable to measure this directly.
Using only this qualitative, ``graphical'' knowledge, our results establish how potential incentives for ITV and PSIE can be assessed.
A weakness is that incentives can only be excluded, not confirmed. 
For the latter task, access to the data distribution is needed.

\paragraph{Measuring introduced unfairness.} When we have access to the model's outputs, we may wish to measure its introduced unfairness. This is usually possible for ITV given appropriate data, as it is defined in terms of conditional probabilities, which can be easily estimated if the variables are observed.
Measuring PSIE is often more challenging, as it is defined in terms of PSEs, whose calculation typically require knowledge of the causal graph, and sometimes even the exact structural functions. The exact conditions for identifying the PSEs
are given by Theorems 4 and 5 of \citet{avin2005identifiability} for Markovian models (i.e. models in which every exogenous variable is independent and in the domain of at most one function $f_V$), and in Theorems 3 and 4 of \citet{shpitser2013counterfactual} for non-Markovian models.
Alternative definitions for PSE can also be used in the PSIE definition. See \cite{shpitser2013counterfactual} for details of a more readily estimated (though less general) variant.

\paragraph{Limitations and risks.}
In addition to the limitations of graphical models (e.g.\ the sensitivity of results to assumptions), our graphical criteria results pertain to \textit{optimal} policies. Trained models may be substantially suboptimal, if the model class is insufficiently powerful, or insufficient training data is used \cite{miller1993loss}. In addition, our graphical criteria give conditions for \textit{compatibility} with some  incentive -- meeting the criteria does not guarantee an incentive for all parameterisations. That said, our empirical results show that an incentive does arise a large proportion of the time.
Similarly, failing to meet the criteria guarantees that optimal policies without the property exist, but does not guarantee this for all optimal policies.  

It is especially important to take account of the limitations of fairness measures because if inappropriately applied, they could cause a 
failure to recognise and address actual injustices.
We highlight four limitations, starting with the most general:
1) Fairness definitions require us to define and formalise
group membership, an exercise that is fraught with practical and ethical difficulties \cite{west2019discriminating,kohler2018eddie, hanna2020towards}.
2) Narrow definitions of unfairness are liable to miss manifestations of injustice, and 
aspects of what we mean by unfairness \citep{kohler2018eddie}.
3) Group fairness definitions may overlook (un)fairness to individuals \citep{dwork2012fairness, kleinberg2016inherent}. 
4) Translating a causal effect into a normative fairness judgement is often complex. While we aim to assist with this as in the discussion around Figure \ref{fig:biased_recruiter_B}, this problem is far from resolved.

\paragraph{Findings.} In this paper we have proposed new definitions for introduced total variation (ITV) and introduced path-specific effects (PSIE) for supervised learning models, and established their graphical criteria. %
Key takeaways include:
\begin{itemize}
    \item \textbf{Models can be incentivised to introduce unfairness.} A predictor can be unfair in spite of fair labels if a feature it depends on is statistically dependent on the sensitive attribute $A$. In the case of P-admissible loss, disparity may still be introduced if the predictor needs to know $A$ to properly interpret its features.  
    \item \textbf{Incentives depend on the loss function and the features.} In some scenarios in which introduced total variation is incentivised, replacing the loss function with a P-admissible loss function is enough to remove the incentive. If additionally the sensitive attribute is (made) available as a feature, an incentive for introduced total variation is always avoided.
    ~\looseness=-1 %
    \item \textbf{Path-specific introduced effects can help labelling paths as fair or unfair.} 
    A path from $A$ to $\pred$ that looks fair at a first glance, may no longer seem fair if it is revealed that it carries an unwanted amplified effect.
    \item \textbf{It is difficult to rule out introduced disparity/effects.} The graphical criteria for ITV and PSIE are easily met.
    \item \textbf{Fair training labels do not always yield a fair model.}
\end{itemize}

\pagebreak
\section*{Acknowledgements}
We would like to thank for their comments, help, and discussions, Ben Coppin, James Fox, Lewis Hammond, William Isaac, Zac Kenton, Ramana Kumar, Claudia Shi, and Chris van Merwijk.
This work was supported in-part by the Leverhulme Centre for the Future of Intelligence, Leverhulme Trust, under Grant RC2015-067.
\bibliography{biblio}

\clearpage

\appendix

\section{Introduced Mutual Information, Separation and Sufficiency}
\label{sec:app_info_theory}

In this section, we present a version of Proposition \ref{lem:sep}
that is generalised to the case where the 
sensitive variable $A$, 
label $Y$, 
and prediction $\preds$
all have categorical (or even continuous) domain.
Moreover, we do not 
strictly require separation or sufficiency.
Rather, we only require 
a bound on the degree to which 
separation or sufficiency is violated.

\subsection{Background}
First, we recap four basic concepts from information theory.

Let $X,Y$ be a pair of discrete random variables, with the probability density function $p(x,y)$. Then the \emph{entropy} is given by $$H(X,Y)=-\sum_{x,y}p(x,y)\log p(x,y),$$
the \emph{conditional entropy} is given by
\begin{equation*}
\label{eqt:cond_entropy}
H(X\mid Y)=\sum_{x,y} p(x,y)\log p(x\mid y),
\end{equation*} 
and the \emph{mutual information} is given by
\begin{equation}
\label{eqt:mutual_information}
I(X;Y)=H(X)-H(X\mid Y).
\end{equation}
For variables $X,Y,Z$, with density $p(x,y,z)$, the \emph{conditional mutual information} is given by
$$I(X;Y\mid Z)=H(X\mid Z)-H(X\mid Y,Z).$$

\subsection{Introduced Mutual Information}
Proposition \ref{lem:sep} bounds introduced total variation,
which is only defined for discrete-valued sensitive attributes.
A similar notion that will generalise to 
continuous case is mutual information.
When applied to a prediction $\preds$, mutual information  
can quantify unfairness, and
has been termed the \emph{independence gap}
$\indgap=I(\preds;A)$. 
When applied to the label $Y$, the mutual information $I(Y;A)$
depends only on the data generating process up until now, 
so has been termed the \emph{legacy}
$\legacy=I(Y;A)$ \citep{hertweck2021gradual}.
By substracting the legacy from the independence gap, 
we can obtain a measure of introduced unfairness.
\begin{definition}[Introduced Mutual Information]
Let $Y$ and $\hat{Y}$ respectively be real-valued target and predicted 
label variables. Then, the \emph{introduced mutual information} (IMI) is 
$$\IMI=I(\hat{Y}; A)-I(Y;A).$$
\end{definition}

To bound the IMI, we will also use gradual versions of sufficiency and separation:
The degree of violation of separation is called 
the \emph{separation gap} $\sepgap=I(\hat{Y};A\mid Y)$, 
where $I(B;C \mid D)$ denotes the mutual information
of $B$ and $C$, conditional on $D$ \citep{hertweck2021gradual}. 
The sufficiency gap is $\sufgap=I(Y;A\mid \hat{Y})$.
If the separation gap (or sufficiency gap) is equal to zero,
then separation (sufficiency) is achieved.

\begin{proposition}[IMI, separation and sufficiency] \label{prop:IMI-sep-suff}
Let $Y$ be the label,
$\hat{Y}$ be the prediction, 
and $A$ be some sensitive attribute, 
each of which is a random variable with discrete domain. Then
$$\indgap - \legacy = \IMI= \sepgap - \sufgap.$$
\end{proposition}

\begin{proof}
The left-hand equality is immediate from definitions. The right-hand equality is proved as follows:
\begin{align*}
    &=\IMI \\
    &=I(\hat{Y};A)-I(Y;A) \\
    &= I(A;Y,\hat{Y}) - I(Y; A \mid \hat{Y}) - (I(A;\hat{Y},Y) - I(\hat{Y}; A \mid Y)) \\
    \intertext{\hfill \text{(Chain rule and symmetry of mutual information)}}
    &= I(\hat{Y}; A \mid Y) - I(Y; A \mid \hat{Y})  \\
    &= \sepgap - \sufgap. \qedhere
\end{align*}
\end{proof}

This proposition has two immediate consequences.
Firstly, it gives an alternative proof of why independence 
and sufficiency are mutually incompatible:
if the legacy is equal to $k \in \mathbb{R}$, 
then the sum of independence 
and sufficiency must be at least $k$, 
because separation is non-negative.

Secondly, it implies that the introduced mutual information 
can be bounded by the separation and sufficiency.
Since mutual information is non-negative, it follows that 
$-\sufgap \leq \IMI \leq \sepgap$.
If we know that the separation gap is bounded above by $m \in \mathbb{R}$, then so too is the IMI.
Likewise, if we know that the sufficiency gap is bounded above by $n \in \mathbb{R}$, 
then the IMI
is bounded below by $-n$.

To prove \cref{prop:IMI-sep-suff} for continuous variables,
replace the definition of mutual information given in equation (\ref{eqt:mutual_information}) with 
the definition for continuous random variables (see Definition 1.6.2, \citet{ihara1993information}),
and follow the identical steps of the proof above.
In particular, note that the conditional mutual 
information is symmetric, and that the chain rule 
for mutual information holds in the continuous case, 
just as in the discrete case
\citep[Theorem 1.6.3 - I.6]{ihara1993information}.

\section{Proof for ITV Incentives}
\label{sec:app_itv_proofs}

\theoremitvgc*

\begin{proof}[Proof of completeness]

We will adapt the construction used to prove Theorem 12 in \citet[Appendix C.2]{Everitt2021agent}.

Suppose the graphical criteria holds. That is, $A$ is (unconditionally) d-connected to $W$ via some path $p_{AW}$, and $W$ is requisite, i.e.\ is d-connected to $\target$ conditioned on ($\Payhat \cup \hat{Y} \setminus W$), via some path $p_{W\target}$.

Since $A$ is unconditionally d-connected to $W$, we have that $p_{AW}$ is of the form $p_{AW}: A \dashleftarrow T \dashrightarrow W$ where it may be that $A=T$ or $W=T$. We will handle the case that $T=W$ below (in case 3), so we may now assume that $T$ and $W$ are separate nodes. Let $T,A,W$ all have domain $\{0,1\}$, and let $T\sim\textrm{Bern}(0.9)$ (that is, $T=1$ with probability 0.9, $T=0$ otherwise), and let $T\pathto A$ be a \emph{copying path} where each node $V \ne T$ on the path copies the value of its parent $X$ on the path, i.e.\ $f_V(X)=X$.
This means that $P(A=T) = 1$.
We will return to the relationship between $W$ and $T$ later in the proof.\\

\noindent \emph{Case 1: $p_{WY}$ is directed.}\\
If $p_{WY}$ is directed, then we can generalise the construction from version 2 of the hiring test prediction example (Figure \ref{fig:zero_one}, Section \ref{sec:defining-introduced}).
Let the path $T\pathto W$ be a copying path, so $P(W=T)=1$.
Parameterise the path $p_{WY}$ so that $P(Y=1\mid W=w) = 0.49$ if $w=0$ and $0.51$ otherwise. That is, let $Y \sim Bern(0.49 + 0.02W).$
This means the optimal predictor under zero-one loss will be $\pred = W$. See \cref{fig:gc_itv_case1} for an illustration.

Since $P(A= T = W=\pred) = 1$, we have that $|\ATV(\pred)| = 1$.
Meanwhile, $|\ATV(Y)|=0.02$, which gives $\ITV=0.98>0$.\\

\begin{figure}[ht!]
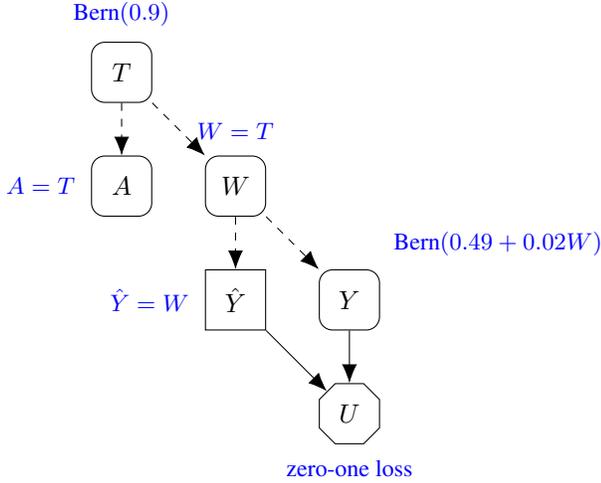

    \centering
    \begin{influence-diagram}
    
    \setrectangularnodes
    
    \node (W) {$W$};
    \node (A) [left = of W] {$A$};
    \node (T) [above = of A] {$T$};

    \node (D) [below = of W, decision] {$\hat{Y}$};
    \node (Y) [right = of D] {$Y$};
    \node (U) [below = of Y, utility] {$U$};

    \edge {D,Y}{U};
    \edge [dashed] {T}{A,W};
    \edge [dashed] {W}{Y,D};
    \begin{scope}[
        every node/.style = {draw=none, rectangle, node distance=1mm, color=blue}]
        \footnotesize
        \node [below = of U] {zero-one loss};
        \node [left = of A] {$A=T$};
        \node [above = of T] {Bern$(0.9)$};
        \node [above = of W] {$W=T$};
        \node [left = of D] {$\pred = W$};
        \node [above right = of Y] {Bern$(0.49+0.02W)$};
    \end{scope}
    \end{influence-diagram}
    \caption{Proof of \cref{theorem:itv_gc}, Case 1 illustration}
    \label{fig:gc_itv_case1}
\end{figure}

\noindent \emph{Case 2: $p_{WY}$ is not directed}\\
In this case, the conditionally active path $p_{WY}$ must have the following form, where each $C^i$ has a descendent $O^i$ which is a feature (that is, $O^i \in \Payhat$). We allow $m=0$ and for $Y$ and $S^m$ to be the same node, but $W$ and $S^m$ must be distinct.

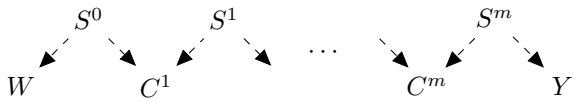
\begin{center}
\begin{tikzpicture}[scale=0.9]
  \node at (0, 0) (W) {$\riparvar$};
  \node at (1, 1) (S0) {$\wusrc^0$};
  \draw (W) edge[<-, dashed] (S0);
  \node at (2, 0) (O1) {$\wucol^1$};
  \draw (S0) edge[->, dashed] (O1);
  \node at (3, 1) (S1) {$\wusrc^1$};
  \draw (O1) edge[<-, dashed] (S1);
  \node[minimum height = 1.5em, minimum width = 2em, draw=none] at (4, 0)
    (O2) {};
  \draw (S1) edge[->, dashed] (O2);
  \node at (4.5, 0.5) {$\cdots$};
  \node[minimum height = 1.5em, minimum width = 2em, draw=none] at (5, 1)
    (Sm1) {};
  \node at (6, 0) (Om) {$\wucol^m$};
  \draw (Sm1) edge[->, dashed] (Om);
  \node at (7, 1) (Sm) {$\wusrc^m$};
  \draw (Om) edge[<-, dashed] (Sm);
  \node at (8, 0) (U) {$\target$};
  \draw (Sm) edge[->, dashed] (U);
\end{tikzpicture}
\end{center}

\noindent \textit{Case 2a: Paths $p_{AW}$ and $p_{W Y}$ only intersect at $W$.}\\
The case is illustrated in \cref{fig:gc_itv_case2a} for $m=0$ and with $Y$ and $S^m$ different nodes.
Let
\begin{itemize}
    \item All variables have a domain that is a subset of $\{-1,0,1 \}$.
    \item $P(S^m=1)=0.6, P(S^m = -1)=0.4$.
    \item Let $S^m \dashrightarrow Y$ be a copying path ($f_V(X)=X$ for all variables $V$ on this path except $S^m$).
    \item $S^i \sim \Unif\{-1,1\}, i \ne m.$
    \item $f_U(\preds,Y) = \syhat.Y,$ so $\syhat$ aims to reproduce $Y$.
    \item ${C^i} = S^{i-1}\cdot S^i$ for all $i$.
    \item ${O^i} = C^i$ for all $i$. The $O^i \in \Payhat$ are observed.
    \item As described above, $T\sim Bern(0.9),$ and $A=T$.
    \item $f_W = A\cdot S^0,$ so $W$ reveals $S^0$ when $A=1$, else $W$ is uninformative when $A=0$.
    \item All variables in $\mathcal{G}$ not mentioned or on the paths described, are 0.
\end{itemize}

\begin{figure}[ht!]
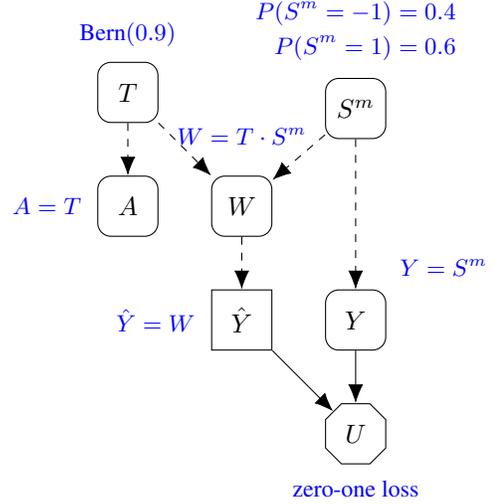

    \centering
    \begin{influence-diagram}
    
    \setrectangularnodes
    
    \node (W) {$W$};
    \node (A) [left = of W] {$A$};
    \node (T) [above = of A] {$T$};
    \node (S) [above = 2 of Y] {$S^m$};

    \node (D) [below = of W, decision] {$\hat{Y}$};
    \node (Y) [right = of D] {$Y$};
    \node (U) [below = of Y, utility] {$U$};

    \edge {D,Y}{U};
    \edge [dashed] {T}{A,W};
    \edge [dashed] {W}{D};
    \edge [dashed] {S} {W,Y};
    \begin{scope}[
        every node/.style = {draw=none, rectangle, node distance=1mm, color=blue}]
        \footnotesize
        \node [below = of U] {zero-one loss};
        \node [left = of A] {$A=T$};
        \node [above = of T] {Bern$(0.9)$};
        \node [above = of W, yshift=2mm] {$W=T\cdot S^m$};
        \node [left = of D] {$\pred = W$};
        \node [above right = of Y] {$Y=S^m$};
        \node [above = of S] {$\begin{aligned}
            P(S^m=-1) &= 0.4 \\ P(S^m=1) &= 0.6
        \end{aligned}$};
    \end{scope}
    \end{influence-diagram}
    \caption{Proof of \cref{theorem:itv_gc}, Case 2a illustration}
    \label{fig:gc_itv_case2a}
\end{figure}

Since each feature $O^i$ only reveals the product of two of the $S^i$, $\syhat$ can only infer $Y$ by multiplying all the features $O^i$ and $W$ together.
Thus all optimal policies must take $f_{\syhat} = W\cdot\prod_i O^i$ (which yields $\syhat = Y$), whenever $W \in \{-1,1\}$. When $W=0$ , then the features are completely uninformative, so to optimise zero-one loss every optimal policy must take $f_{\preds}(W=0, \cdot) = 1,$ (the most likely value of $Y$).

Since $Y$ is independent of $A$, we have $\mathbb{E}(Y \mid A=0) = \mathbb{E}(Y \mid A=1)=0.2,$ and so $\ATV(Y)=0.$
The only optimal policy satisfies $\mathbb{E}(\preds \mid A=0)=1, \mathbb{E}[\preds \mid A=1]= \mathbb{E}(\target \mid A=1)=0.2$, and so $|ATV(\preds)|=0.8.$
Thus $\ITV(\preds)=0.8>0.$\\

\noindent \textit{Case 2b: Paths $p_{AW}$ and $p_{W\preds}$ intersect at more than $W$.}\\
We split this into two sub-cases:\\

\noindent \textit{Case 2b.i:}\\
If $A$ is not on $p_{W\preds},$ then for the nodes $V \ne W$ on $p_{AW}$ that do intersect, let $V=(V_1,V_2),$ where $V_1$ plays the role as per $p_{AW}$ in case 1, and $V_2$ plays the role as per $p_{W Y}$ in case 1. Then we may proceed as in case 1.\\

\noindent \textit{Case 2b.ii:}\\
Conversely, suppose that $A$ is on $p_{W Y}$. If $A$ occurs between $W$ and some collider $C^i$ on $p_{W Y},$ then repeat case 1 with feature $O^i$ (the observed descendent of $C^i$) playing the part of $W$.

Else if $A$ occurs after $C^m,$ then the graph must contain the structure given in Figure \ref{fig:gc_itv_case2ii}.

\begin{figure}[ht!]
    \centering
    \begin{influence-diagram}
    
    \setrectangularnodes
    
    \node (help) [phantom] {};
    \node (A) [above = of help] {$A$};
    \node (O) [left = of help] {$O^m$};
    \node (S) [right = of help] {$S^m$};

    \node (Y) [below = of S] {$Y$};
    \node (D) [below = of O, decision] {$\hat{Y}$};
    \node (U) [below = 2 of help, utility] {$U$};

    \edge [dashed] {D,Y}{U};
    \edge [dashed] {O}{D};
    \edge [dashed] {S}{A,Y};
    \edge [dashed] {A}{O};
    \begin{scope}[
        every node/.style = {draw=none, rectangle, node distance=1mm, color=blue}]
        \footnotesize
        \node [below = of U] {zero-one loss};
        \node [left = of D] {$f_{\hat{Y}}(O^m) = O^m$};
        \node [right = of Y] {$Y=\begin{cases}
        0 & \text{with } p = \tfrac{1}{4}\\
        1 & \text{with } p = \tfrac{1}{4}\\
        A & \text{with } p = \tfrac{1}{2}\\        
        \end{cases}$};
        \node [right = of S] {$S^m\sim \text{Bern}(0.5)$};
        \node [left = of O] {$O^m=A$};
        \node [right = of A] {$A = S^m$};
    \end{scope}
    \end{influence-diagram}
    \caption{Completeness proof of \cref{theorem:itv_gc}, Case 2b.ii}
    \label{fig:gc_itv_case2ii}
\end{figure}

In which case, let
\begin{itemize}
    \item $S^m \sim \textrm{Bern}(0.5)$.
    \item $S^m \dashrightarrow A \dashrightarrow O^m$ be a copying path, so that $P(S^m =A)=1.$
    \item $P(Y=0) = P(Y=1)= 0.25,$ $P(Y=A) = 0.5.$ That is, with probability 0.5, $Y$ copies the value of $A$, otherwise $Y$ takes value 0 or 1 equiprobably, independently of $A$.
    \item $U$ uses zero-one loss.
\end{itemize}
Then the optimal policy will be $f_{\preds}(O^m) = O^m = A,$ and we have $|\ATV(\preds)|=1,$ $|\ATV(Y)|=0.75,$ and so $\ITV = 0.25>0.$\\

\noindent \textit{Case 3:}\\
Finally, we return to the case $T=W$. Since the path $p_{W Y}$ contains a collider, then let $O^1$ play the part of $W$ and proceed as above.
\end{proof}

\section{Proof for PSIE Incentives}
\label{sec:app_psie_proofs}

\theorempsiegc*

\begin{proof}[Proof of completeness]
Suppose $\mathcal{G}$ satisfies the graphical criteria.
Then $\mathcal{G}$ also satisfies the Theorem \ref{theorem:itv_gc} graphical criteria (for ITV incentives), with the added constraint that the d-connected path $p$ (given by (i)) is directed, and in $\mathcal{P}$. Let the d-connected path between $W$ and $Y$ guaranteed by condition (ii) be denoted $p'$.

One of the constructions in the Theorem \ref{theorem:itv_gc} completeness proof must be compatible with $\cid$, with $p'$ playing the role of $p_{W Y}$ and $p$ playing the role of $p_{AW}$. Since these constructions have only one (non-zero) path between $A$ and $\target$ (which must be $p'$ by assumption), it follows that either $\PSE_\mathcal{P}(Y) = 0$ (if $p' \notin \mathcal{P}$), or $\PSE_\mathcal{P}(Y) = \ATV(Y).$  
Since these constructions have at most one path between $A$ and $\preds,$ (which must be $p$ by assumption), it follows that intervening via $\mathcal{P}$ and conditioning on $A$ will have the same effect on $\preds$. Thus $\PSE_\mathcal{P} = \ATV(\preds).$

Thus $\PSIE_{\mathcal{P}}\ge \ITV > 0.$
\end{proof}

\end{document}